\documentclass{article}

\usepackage[nonatbib,preprint]{nips_2018}

\usepackage[utf8]{inputenc} 
\usepackage[T1]{fontenc}    
\usepackage{hyperref}       
\usepackage{url}            
\usepackage{booktabs}       
\usepackage{amsfonts}       
\usepackage{nicefrac}       
\usepackage{microtype}      
\usepackage{floatrow} 

\usepackage{color}
\usepackage{algpseudocode}
\usepackage{algorithm}
\usepackage[numbers,sort&compress]{natbib}
\usepackage[pdftex]{graphicx}
\usepackage{tabularx}
\usepackage{times}
\usepackage{amsmath}
\usepackage{amssymb}
\usepackage{amsthm}
\usepackage{booktabs}
\usepackage{subcaption}
\definecolor{linkcolour}{rgb}{0,0.2,0.6}
\hypersetup{colorlinks,breaklinks,linkcolor=linkcolour,citecolor=linkcolour,filecolor=linkcolour,urlcolor=linkcolour}
\usepackage{bm}
\usepackage{multirow} 

\newtheorem*{theorem*}{Theorem}

\newtheorem*{lemma*}{Lemma}
\newtheorem{case}{Case}

\newfloatcommand{capbtabbox}{table}[][\FBwidth]

\def\CC{\textcolor{red}}
\def\cc{\textcolor{red}}
\def\cb{\textcolor{blue}}

\title{Rotation Equivariance and Invariance in Convolutional Neural Networks}

\author{
  Benjamin Chidester\\
  Computational Biology Department\\
  School of Computer Science\\
  Carnegie Mellon University\\
  Pittsburgh, PA 15213 \\
  \And
  Minh N.~Do\\
  Department of Electrical and Computer Engineering\\
  University of Illinois at Urbana-Champaign\\
  Urbana, IL 61801\\
  \And
  Jian Ma\\
  Computational Biology Department\\
  School of Computer Science\\
  Carnegie Mellon University\\
  Pittsburgh, PA 15213 \\
  \texttt{jianma@cs.cmu.edu}\\
}

\begin{document}
\maketitle

\begin{abstract}
\noindent 
Performance of neural networks can be significantly improved by encoding known invariance for particular tasks.
Many image classification tasks, 
such as those related to cellular imaging, 
exhibit invariance to rotation.
We present a novel scheme using the magnitude response of the 2D-discrete-Fourier transform (2D-DFT) to encode rotational invariance in neural networks, along with a new, efficient convolutional scheme for encoding rotational equivariance throughout convolutional layers.
We implemented this scheme for several image classification tasks and demonstrated improved performance, in terms of classification accuracy, time required to train the model, and robustness to hyperparameter selection, over a standard CNN and another state-of-the-art method.
\end{abstract}

\def\CC{\textcolor{red}}
\def\cc{\textcolor{red}}
\def\cb{\textcolor{blue}}

\section{Introduction}
\label{sec:intro}

Though the appeal of neural networks is their versatility for arbitrary classification tasks, there is still much benefit in designing them for particular problem settings.
In particular, their effectiveness can be greatly increased by encoding invariance to uniformative augmentations of the data~\citep{LeCun1989}.
The wide success of convolutional over fully-connected neural networks for image classification tasks is due to its innovation of encoding local translation invariance by convolution and pooling operations.
If such invariance is not explicitly encoded, the network must learn it from the data, requiring more parameters and thereby increasing its susceptibility to overfitting.

A key global invariance inherent to several computer vision settings, including satellite imagery and all forms of microscopy imagery, is {\it rotation}~\citep{Cheng2016,Boland2001}.
To aid neural networks in learning in such settings, standard practice is to augment the training data by rotations.
Recently, new formulations of convolutional layers have been proposed for neural networks, including a spatial transform layer~\citep{Jaderberg2015} and a deformable convolutional layer~\citep{Dai2017}, that allow the network to learn non-regular sampling patterns and can aid in learning rotation invariance, though invariance is not explicitly enforced.
\citet{Cheng2016} recently proposed a means of encouraging a network to learn global rotation invariance and showed improved performance on satellite imagery detection tasks, but the invariance is not expressly encoded.
Additionally, the convolutional layers of the network do not maintain the property of rotation equivariance with the input image, which requires that the network learn this equivariance and could therefore hinder performance.
In all such methods, it is not guaranteed that the network will truly learn such invariance and therefore generalize properly to unseen data, especially for tasks where imaging data is scarce.
For applications such as cellular microscopy image analysis, this is a prevalent challenge.

There are a variety of ways in which rotation invariance can be explicitly encoded into a neural network.
The earliest works proposed methods for fully-connected neural networks, but did not consider how to extend them to convolutional networks~\citep{Widrow1988,Khotanzad1990,Fukumi1997,Rowley1998}.
Recently, several methods have been proposed for CNNs.
\citet{Dieleman2016} proposed the use of rotated filters for convolutional layers, with subsequent pooling operations to encode invariance.
However, by imposing an invariant transform immediately after a convolutional layer, only local invariance can be captured and some global variational structure will be lost.
\citet{Cohen2016} and \citet{Worrall2017} both considered local and global rotation equivariance in convolutional networks.
The architecture proposed in~\citep{Cohen2016}, called a group-equivariant convolutional neural network (G-CNN), maintains the property of equivariance to any group, including rotation and flips, throughout the convolutional layers of the network and pools across all groups for each filter to encode invariance.
Improvement in classification accuracy of G-CNN over a standard CNN, as well as other state-of-the-art methods, was shown in several problem settings, including the classification of rotated images of MNIST.
By delaying the imposition of invariance till the final, fully-connected layers, as much global structure as possible is thereby incorporated.
However, by applying pooling on individual filters, potentially valuable mutual rotational information across filter responses is therefore lost.

In this paper, for an architecture that better captures global rotation invariance, we propose two innovations for CNNs:
(1) A novel, efficient formulation of convolution for maintaining rotation equivariance.
In this formulation, to achieve equivariance, rather than convolving each filter across the entire image, rotated filters are convolved along radial, conic regions of the input feature map.
(2) The formulation of a new transition layer between convolutional and fully-connected layers to encode invariance to rotation of the preceding convolutional layers, using the magnitude response of the 2D-discrete-Fourier transform (2D-DFT).
This transition layer transforms rotations of feature maps into circular shifts, to which the magnitude response of the 2D-DFT is invariant, in the transformed space.
Unlike the pooling of each filter response individually, our invariant transform is able to preserve valuable, mutual rotational information between different filter responses.
This insight was leveraged in earlier work for texture classification using wavelets~\citep{Do2002,Jafari-Khouzani2005,Ojala2002,Charalampidis2002}, but has not previously been integrated into the architecture of a CNN.
We refer to a network composed of both of these innovations as a \emph{rotation}-\emph{invariant} CNN (RiCNN).

To demonstrate the effectiveness of these contributions for rotation-invariant computer vision tasks, we implemented variations of network architectures both with and without each enhancement and trained them for several related problems in different applications: classifying rotated MNIST images, classifying our own synthetic images that model biomarker expression in microscopy images of cells, and localizing proteins in budding yeast cells~\citep{Kraus2017}.
Our analysis clearly shows that adding the magnitude response of the 2D-DFT to encode rotational invariance significantly improves the classification accuracy across these diverse data sets, while also reducing the time required to train the networks.
It also shows that our proposed rotation-equivariant convolution formulation improves classification accuracy generally over the standard raster convolution formulation and over the equivariant method of G-CNN in some settings.
Code for the implementation of RiCNN, along with code to recreate the results in the paper, is available at: \href{https://github.com/bchidest/RiCNN}{https://github.com/bchidest/RiCNN}.

\section{Rotation-Invariant CNN Formulation}

The overall architecture of RiCNN, as shown in Fig.~\ref{fig:ricnn_diagram}, consists of three stages: 
(1) rotation-equivariant convolutional layers with conic convolutional regions; 
(2) a rotation-invariant transition layer from convolutional to fully-connected layers using the 2D-DFT; 
and (3) fully-connected layers and output layer.

\subsection{Rotation-Equivariant Convolutional Layers}


Our proposed formulation for a rotation-equivariant convolutional layer associates rotations in the input image with rotations in the feature map that is generated for each filter.
This is accomplished by convolving the input with each filter, rotated by multiples of $\frac{\pi}{2R}$, for $R\in\mathbb{Z}_{>0}$, over corresponding conic regions of the domain.
We consider feature maps as functions over 2D space and denote the input feature map from a previous layer, or the input of the network, by $a\colon\mathbb{Z}^{2} \to \mathbb{R}^{d}$, where $d$ is the depth of the feature map.
The domain is partitioned into conic regions $\left\{\mathcal{C}_{r}\right\}_{r=0}^{4R-1}$ emanating from the center and the borders between conic regions $\left\{\mathcal{B}_{r}\right\}_{r=0}^{4R-1}$.
Although equivariance in this formulation will only be guaranteed for rotations of multiples of $\frac{\pi}{2}$ (i.e. $R=1$), rotations of smaller degree can be approximated while still maintaining equivariance for rotations of multiples of $\frac{\pi}{2}$.
The conic regions and boundaries are defined by:
\begin{align}
    \mathcal{C}_{r} &= \left\{(x,y)\in \mathbb{Z}^{2}\colon \theta_{r} < \operatorname{arccot}\left(x/y\right) + \pi\mathbb{I}(y < 0) < \theta_{r + 1}\right\},\\
    \mathcal{B}_{r} &= \left\{(x,y)\in \mathbb{Z}^{2}\colon \operatorname{arccot}\left(x/y\right) + \pi\mathbb{I}(y < 0) = \theta_{r} \right\},
\end{align}
where $\theta_{r} = \frac{2\pi r}{4R}$ and $\mathbb{I}(\cdot)$ is the indicator function.
Note that the origin does not have a defined output for $\operatorname{arccot}(\cdot)$ and therefore does not belong to any region or boundary, which will be addressed subsequently.
The input feature map $a$ is convolved over each region of the domain with each filter $\omega_{k}\colon\mathbb{Z}^{2}\to \mathbb{R}^{d}$, of $k\in\left\{0,1,\ldots,K-1\right\}$, but this filter must be rotated by the appropriate angle $\theta_{r}$ for each region or boundary.
The operation $\Theta\colon f(\mathbb{Z}^{2})\times \left[0, 2\pi\right] \to f(\mathbb{Z}^{2})$, where $f(\mathbb{Z}^{2})$ is the set of functions on $\mathbb{Z}^{2}$, rotates the filter $\omega_{k}$ counter-clockwise by an angle $\theta\in\left[0, 2\pi\right]$ according to some interpolation scheme, such as nearest values or bilinear.
In this work, we consider only filters for which the depth $d$ is equal to the depth of the input feature layer.
We denote the convolution of the feature map with a particular filter and rotation by $\phi_{r}\colon f(\mathbb{Z}^{2})\times f(\mathbb{Z}^{2})\to f(\mathbb{Z}^{2})$, which is defined as:
\begin{equation}
    \phi_{r}(a, \omega_{k})(x,y) \triangleq (a \ast \Theta(\omega_{k}, \theta_{r}))(x,y) = \sum_{(x^{\prime}, y^{\prime})\in\mathbb{Z}^{2}}\Theta(\omega_{k}, \theta_{r})(x^{\prime}, y^{\prime})a(x - x^{\prime}, y - y^{\prime}).
\end{equation}

Consideration must be given to which rotations will be applied along boundaries and at the origin to maintain the desired property of equivariance.
We chose to handle boundaries by rotating the filter by the respective angles of the conic regions that are separated by the boundary and then pooling over the results of the convolution for both rotations.
We used max pooling, though other pooling operators, such as the average, could also be used.
The result of convolution with the $k$-th filter is denoted by $\phi\colon f(\mathbb{Z}^{2})\times f(\mathbb{Z}^{2})\to f(\mathbb{Z}^{2})$ and given by:
\begin{equation}
\begin{split}
    \phi(a, \omega_{k})(x,y) \triangleq 
    \begin{cases}
        \phi_{r}(a,\omega_{k})(x,y) & (x,y)\in\mathcal{C}_{r},\\
        &r\in\left\{0,1,\ldots,4R-1\right\},\\
        \max \left\{\phi_{r}(a,\omega_{k})(x,y), \phi_{r+1}(a,\omega_{k})(x,y)\right\}& (x,y)\in\mathcal{B}_{r},\\
        &r\in\left\{0,1,\ldots,4R-1\right\},\\
        \max_{r\in\left\{0,1,\dots, R-1\right\}} \phi_{r}(a,\omega_{k})(x,y)& x=0,y=0.
    \end{cases}
\end{split}
\label{eqn:cases}
\end{equation}

\begin{figure*}[t!]
  \centering
  \centering
  \includegraphics[width=.97\textwidth]{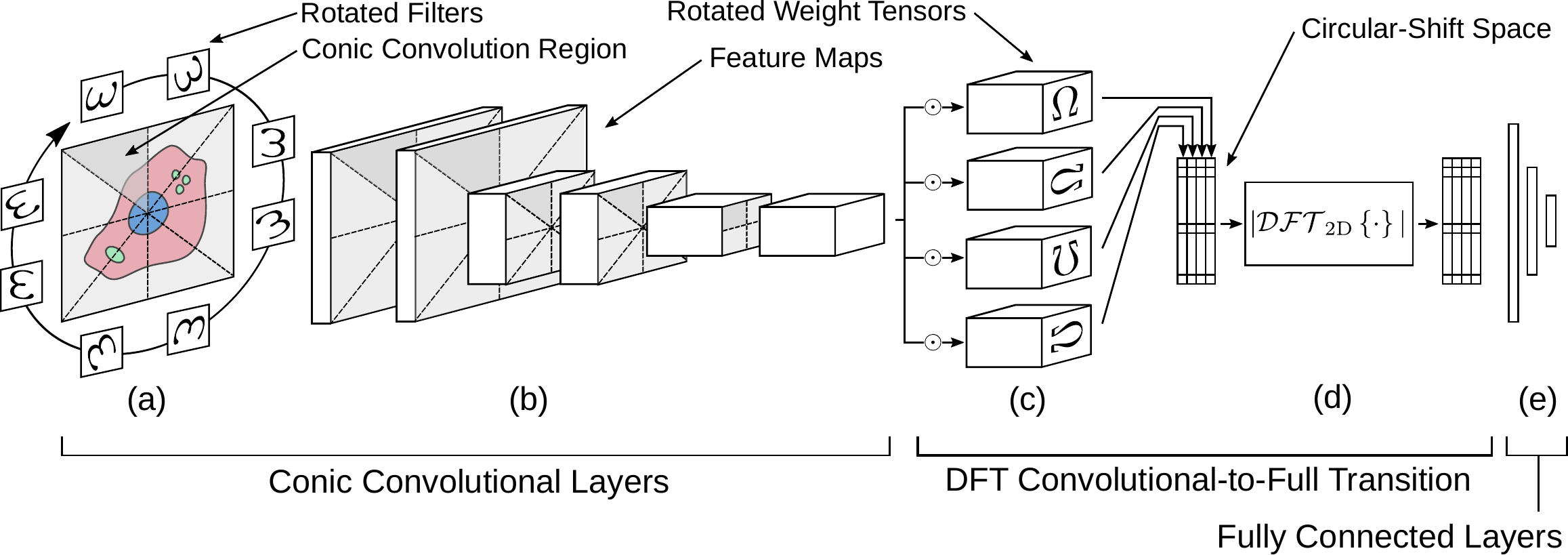}
  \caption{The overall architecture of the proposed rotation-invariant CNN.
      (a) Filtering the image by various filters $\left\{\omega\right\}$ at rotations in corresponding conic regions preserves rotation-equivariance.
      (b) Subsequent convolutional feature maps are filtered similarly.
      Rotation-invariance is encoded by the transition from convolutional to fully-connected layers, which consists of (c) element-wise multiplication and sum, denoted by $\odot$, with rotated weight tensors $\left\{\Omega\right\}$, transforming rotation to circular shift, and (d) application of the magnitude response of the 2D-DFT to encode invariance to such shifts.
      (e) This output is reshaped and passed through the final, fully-connected layers.}
  \label{fig:ricnn_diagram}
\end{figure*}

It is typical in convolutional networks to perform subsampling, either by striding the convolution or pooling local regions, to reduce the dimensionality of subsequent layers.
Given that the indices of the plane of the feature map are in $\mathbb{Z}^{2}$ and are therefore centered about the origin, a downsampling of $D\in \mathbb{Z}_{> 0}$ can be applied while maintaining rotational equivariance.
After subsampling, the result is passed through a non-linear activation function $\sigma\colon\mathbb{R}\to\mathbb{R}$, such as ReLU, with an added offset $c_{k}\in\mathbb{R}$.
The composition of these operations is denoted by $\Phi\colon f(\mathbb{Z}^{2})\to f^{K}(\mathbb{Z}^{2})$ and given by:
\begin{equation}
    \Phi(a)(x,y) = \left\{\sigma\left(\phi\left(a,\omega_{k}\right)(Dx, Dy) + c_{k}\right)\right\}_{k=0}^{K-1}.
\end{equation}
Example convolutional regions with appropriate filter rotations are shown in Fig.~\ref{fig:ricnn_diagram}.
We note that $R$ need not be the same for each layer to maintain rotation equivariance to rotations of $\frac{\pi}{2}$ throughout the network, and it may be advantageous to use a finer discretization of rotations for early layers, when the feature maps are larger, and gradually decrease $R$.

An advantage of this formulation is its efficient use of storage and computation.
In theory, the only additional computation and storage required by the proposed formulation, beyond that of standard convolution, is for the boundaries and origin, though in actual implementation, given the hardware efficiency of GPUs for raster convolution, some additional computation and storage was used.
In contrast, the formulation of G-CNN consists of convolving the entire input feature map with rotations of each filter of multiples of $\frac{\pi}{2}$,
yielding a separate 2D feature map for each rotation of each filter.
A potential disadvantage of this formulation is its possible instability to translations, or jitter, of the input image, which we assess in the experiments presented in Section 3.

We now prove the rotation equivariance property of the proposed convolutional formulation for the case of rotations of $\frac{\pi}{2}$.
First, we note that for such rotations, regardless of the interpolation function $\Theta$, no interpolation is needed, and the operation is merely a change of indices.
Consider $\theta=\frac{n\pi}{2}$, for some $n\in\mathbb{Z}$, and a given feature map $a$.
We define the function $\mu_{n}(x,y)\colon \mathbb{Z}^{2} \to \mathbb{Z}^{2}$ as:
\begin{equation}
   \mu_{n}(x,y) = 
   \begin{cases}
       (x,y) & n\mod 4 = 0,\\
       (-y,x) & n\mod 4 = 1,\\
       (-x,-y) & n\mod 4 = 2,\\
       (y,-x) & n\mod 4 = 3.
   \end{cases}
\end{equation}
Therefore,
\begin{equation}
   \Theta(a, \theta)(x,y) = a(\mu_{n}(x,y)), \hspace{0.3cm}\text{for } \theta=\frac{n\pi}{2} .
\end{equation}
Given this relationship, we establish the following lemma, that the operation $\Theta$ distributes over multiplication.
\begin{lemma*}
   For given functions $a,b\in f(\mathbb{Z}^{2})$ and $\theta\in\left[0, 2\pi\right]$, if $\theta = \frac{n\pi}{2}$ for some $n\in \mathbb{Z}$, then the operation $\Theta\colon f(\mathbb{Z}^{2})\times\left[0,2\pi\right]\to f(\mathbb{Z}^{2})$ can be distributed over multiplication, i.e.,
   \begin{equation}
       \Theta(ab,\theta) = \Theta(a, \theta)\Theta(b,\theta)
   \end{equation}
\end{lemma*}
\begin{proof}
   For $(x,y)\in\mathbb{Z}^{2}$,
   \begin{align*}
       \Theta(a, \theta)(x,y)\Theta(b,\theta)(x,y)
       &= a(\mu_{n}(x,y))b(\mu_{n}(x,y))\\
       &= ab (\mu_{n}(x,y))\\
       &= \Theta(ab, \theta)(x,y)
   \end{align*}
\end{proof}
With this lemma, we prove the rotational equivariance of the proposed CNN architecture.
Note that the theorem is established independent of the number of regions $R$.

\begin{theorem*}[Rotational Equivariance]
    If the input to the operation $\phi \colon f(\mathbb{Z}^{2})\times f(\mathbb{Z}^{2}) \to f(\mathbb{Z}^{2})$ is transformed by a rotation of an angle $\theta=\frac{n\pi}{2}$, for some $n\in\mathbb{Z}$, then the output of the operation will rotate equivalently.
    In other words, for given functions $a, b \colon \mathbb{Z}^{2}\to \mathbb{R}^{K}$, $\forall (x,y) \in\mathbb{Z}^{2}$, 
    \begin{equation}
        \phi(\Theta(a, \theta), b)(x,y) = \Theta(\phi(a, b), \theta)(x,y).  \label{eqn:re_thm}
    \end{equation}  \label{thm:re}
\end{theorem*}
\begin{proof}
   First, we establish the following useful relationship:
   \begin{align*}
       \phi_{r}(\Theta(a,\theta), b)(x,y) &= 
       \sum_{(x^{\prime}, y^{\prime})\in\mathbb{Z}^{2}}\Theta(b, \theta_{r})(x^{\prime}, y^{\prime})\Theta\left(a, \theta\right)(x - x^{\prime}, y - y^{\prime})\\
       &= \sum_{(x^{\prime}, y^{\prime})\in\mathbb{Z}^{2}}\Theta\left(\Theta(b, \theta), \theta_{r-nR}\right)(x^{\prime}, y^{\prime})\Theta\left(a, \theta\right)(x - x^{\prime}, y - y^{\prime})\\
       &= \sum_{(x^{\prime}, y^{\prime})\in\mathbb{Z}^{2}} \Theta\left(b, \theta_{r - nR}\right)(\mu_{n} (x^{\prime}, y^{\prime}))
       a(\mu_{n}(x - x^{\prime}, y -y^{\prime}))\\
       &=  \sum_{(x^{\prime}, y^{\prime})\in\mathbb{Z}^{2}}\Theta\left(b, \theta_{r - nR}\right)(\mu_{n} (x^{\prime}, y^{\prime}))
       a(\mu_{n}(x, y) - \mu_{n}(x^{\prime}, y^{\prime}))\\
       &= \sum_{(u,v)\in\mathbb{Z}^{2}} \Theta\left(b, \theta_{r - nR}\right)(u, v)
       a(\mu_{n}(x, y) -(u, v)))\\
       &= \phi_{r-nR}(a,b)(\mu_{n}(x,y))
   \end{align*}
   Now, we prove Eqn.~\ref{eqn:re_thm} for cases of $(x,y)\in\mathbb{Z}^{2}$:
   \begin{case}
       $(x,y)\in\mathcal{C}_{r}$ for some $r\in\left\{0, 1, \ldots, 4R-1\right\}$
   \end{case}
   \begin{align*}
       \phi(\Theta(a,\theta), b)(x,y) &= \phi_{r}(\Theta(a,\theta), b)(x,y)\\
       &= \phi_{r-nR}(a,b)(\mu_{n}(x,y))\\
       &= \phi(a,b, R)(\mu_{n}(x,y))\\
       &= \Theta\left(\phi(a,b), \theta\right)(x,y)
   \end{align*}
   \begin{case}
       $(x,y)\in\mathcal{B}_{r}$ for some $r\in\left\{0, 1, \ldots, 4R-1\right\}$
   \begin{align*}
       \phi(\Theta(a,\theta), b)(x,y) &= \max\left\{\phi_{r}(\Theta(a,\theta), b)(x,y), \phi_{r+1}(\Theta(a,\theta), b)(x,y)\right\}\\
       &= \max\left\{\phi_{r-nR}(a,b)(\mu_{n}(x,y)), \phi_{r-nR + 1}(a,b)(\mu_{n}(x,y))\right\}\\
       &= \phi(a,b)(\mu_{n}(x,y))\\
       &= \Theta\left(\phi(a,b), \theta\right)(x,y)
   \end{align*}
   \end{case}
   \begin{case}
       $(x,y)=(0,0)$
   \begin{align*}
       \phi(\Theta(a,\theta), b)(x,y) &= \max_{r\in\left\{0,1,\ldots,4R-1\right\}}\phi_{r}(\Theta(a,\theta), b)(x,y)\\
       &= \max_{r\in\left\{0,1,\ldots,4R-1\right\}}\phi_{r - nR}(a, b)(\mu_{n}(x,y))\\
       &= \phi(a, b)(\mu_{n}(x,y))\\
       &= \Theta\left(\phi(a,b), \theta\right)(x,y)
   \end{align*}
   \end{case}
\end{proof}
%

\subsection{Rotation-Invariant Transition using the Magnitude of the 2D-DFT}

The standard practice in designing convolutional networks is to use convolutional layers for the first several layers and then transition to fully connected layers for the remaining.
In a `fully-convolutional' network, convolution and downsampling are applied until the spatial dimensions are eliminated and the resulting feature map of the final convolutional layer is merely a vector, with dimension equal to the number of filters.
Otherwise, the feature map of the final convolutional layer must be reshaped to a vector, resulting in the confusion of the spatial relationship of nodes.
Thus, it is intuitive to encode rotation invariance at the transition between convolutional to fully-connected layers, before this spatial information is lost.
In G-CNN, invariance is encoded at this location, by pooling across groups of the feature map of the final convolutional layer of a fully-convolutional network.

Our proposed method for achieving invariance is also located at this transition, but rather than pooling the filter responses, the filter responses are transformed to a space in which rotation becomes circular shift so that the 2D-DFT can be applied to encode invariance.
The primary merit of the 2D-DFT as an invariant transform is that each output node is a function of every input node, and not just the nodes of a particular filter response, thereby capturing mutual information across responses.
Additionally, the 2D-DFT allows for backpropagation for optimization of the network and has highly efficient implementations, especially for inputs of powers of two.

To describe this transition, we adopt the representation of feature maps as tensors, rather than functions, for ease of notation, especially considering our use of the DFT, which operates on finite-length signals.
We denote the feature map generated by the penultimate convolutional layer by $a\in \mathbb{R}^{M\times M\times d}$, where $M\in\mathbb{Z}_{>1}$.
In a fully-convolutional network, the final convolutional layer is in reality just a fully-connected layer, in which the output $a$ is passed through $K\in\mathbb{Z}_{>0}$ fully-connected filters, denoted by weight tensors $\Omega^{(k)}\in \mathbb{R}^{M\times M\times d}$, $k \in\left\{0,1, \ldots, K-1\right\}$.
However, in our formulation, rotations $\theta_{r}$, $r\in\left\{0,1,\ldots,4R-1\right\}$, $R\in \mathbb{Z}_{>0}$, of the weight tensors are also applied, as diagrammed in Fig.~\ref{fig:ricnn_diagram}.
This can in fact be considered as an instance of our proposed rotation-equivariant convolution, where the only convolution occurs at the origin, but the output is not pooled.

As in the previous subsection, we define a function $\Theta$ to rotate feature maps, though instead operating on tensors, $\Theta\colon \mathbb{R}^{M\times M}\times\left[0, 2\pi\right] \to \mathbb{R}^{M\times M}$.
In particular, for angles $\theta = \frac{n\pi}{2}$ for some $n\in\mathbb{Z}$, 
\begin{equation}
    \Theta(a, \theta)_{x,y} = a_{\mu_{n}(x,y)},
\end{equation}
as before. The output of this transition of convolutional to fully-connected layers, denoted by $z\in \mathbb{R}^{K\times 4R}$, is given by,
\begin{equation}
    z_{k, r} = \langle\Theta(\omega^{(k)}, \theta_{r}), a\rangle,
\end{equation}
where $\langle\cdot, \cdot\rangle$ denotes the inner product.
Note that, for $\theta = \frac{n\pi}{2}$,  for some $n\in\mathbb{Z}$,
\begin{align*}
    \langle\Theta(\omega_{k}, \theta_{r}), \Theta(a, \theta)\rangle
    &= \langle\Theta(\Theta(\omega_{k}, \theta_{r - nR}), \theta), \Theta(a, \theta)\rangle\\
    &= \langle\Theta(\omega_{k}, \theta_{r - nR}), a\rangle\\
    &= z_{k, (r-nR)\operatorname{mod}4R}.
\end{align*}
Thus, rotations of the final convolutional layer $a$ will correspond to circular shifts in $z$ along the second dimension, as desired.
The magnitude response of the 2D-DFT transforms these circular shifts to an invariant space $z^{\prime}\in \mathbb{R}^{K\times 4R}$,
\begin{equation}
    z^{\prime}_{k, i} = \left| \mathcal{DFT}\left\{ z\right\}\right|(k, i) = \left|\sum_{h=0}^{K-1}\sum_{r=0}^{4R-1}z_{h, r}e^{-j2\pi\left(\frac{hk}{K} + \frac{r i}{4R}\right)}\right|.
    \label{eqn:dft}
\end{equation}
This process of encoding rotation invariance corresponds to the `Convolutional-to-Full Transition' in Fig.~\ref{fig:ricnn_diagram}.
The result is then vectorized $z^{\prime}\in\mathbb{R}^{4KR}$ and passed into fully-connected layers that precede the final output layer, as in a standard CNN,
\begin{equation}
    a^{\prime} = f(Wz^{\prime} + c),
\end{equation}
where $W\in \mathbb{R}^{K^{\prime}\times 4KR}$ is the weight matrix of the first fully-connected layer, $c\in\mathbb{R}^{K^{\prime}}$ is the offset, and $K^{\prime}$ is the number of nodes in the first fully-connected layer.

Although the 2D-DFT was the basis for the design of this transition layer and the proposed rotation-equivariant convolutional layer, it can also be integrated into other rotation-equivariant networks, such as G-CNN.
Rotation equivariance in G-CNN is encoded along contiguous stacks of feature maps $a\in \mathbb{R}^{M\times M\times 4d}$ of each filter at four rotations.
To integrate the 2D-DFT, first, the output feature maps are reshaped into a two-dimensional tensor such that the second dimension corresponds to rotations, $z\in \mathbb{R}^{M^{2}d\times 4}$.
Note that in the case of a fully-convolutional network, $M=1$, the first dimension corresponds to the filter and the second dimension to its rotation.
In this way, rotations similarly correspond to circular shifts in the new space.
This representation $z$ is then passed through the 2D-DFT, as in Eqn.~\ref{eqn:dft}, excepting that $K = M^{2}d$, and then subsequently through the first fully-connected layer.

\section{Results}

\subsection{Application to Rotated MNIST}

\begin{figure}
\begin{floatrow}
    \capbtabbox{%
        \resizebox{0.5\textwidth}{!}{\begin{tabular}{lc}
        \toprule
        Algorithm & Test Error (\%)\\
        \midrule
        Schmidt and Roth & 3.98\\
        \midrule
        \citet{Cohen2016} (CNN)& 5.03\\
        \midrule
        \citet{Cohen2016} (G-CNN)& 2.28\\
        \midrule
        RiCNN & 2.33\\
        \midrule
        G-CNN + DFT & 2.00\\
        \bottomrule
        \end{tabular}
    }}{%
        \caption{Comparison of test error on the rotated MNIST data set.}
        \label{tab:rotated_mnist}
    }
    \ffigbox{%
        \centering
        \includegraphics[width=0.14\textwidth]{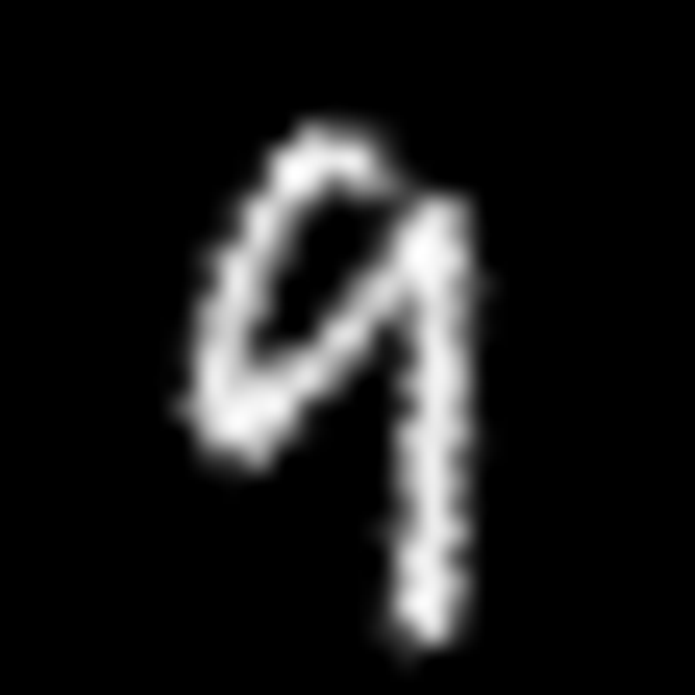}
        \includegraphics[width=0.14\textwidth]{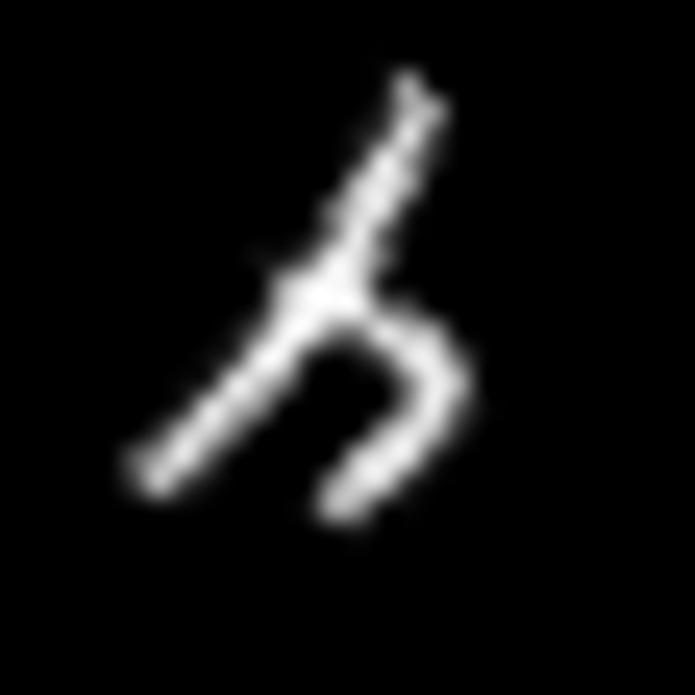}
        \includegraphics[width=0.14\textwidth]{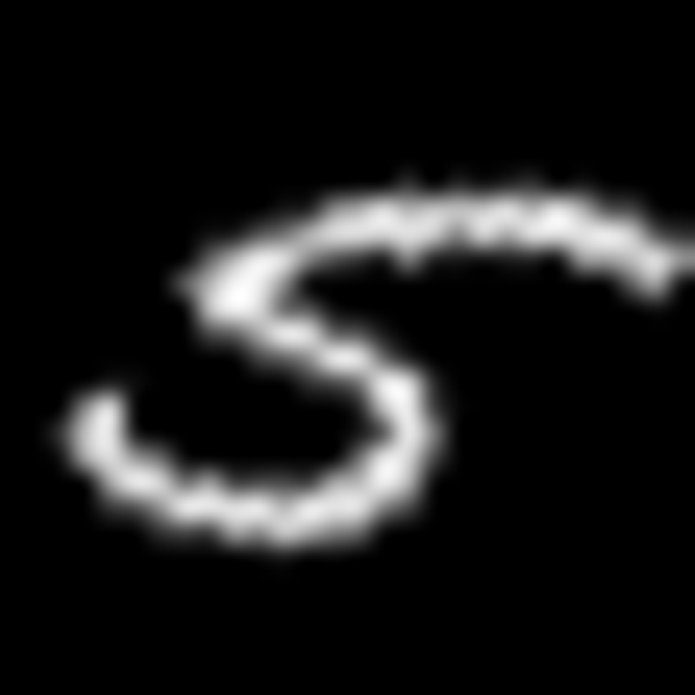}
    }{%
        \caption{Examples from the rotated MNIST data set.}
        \label{fig:mnist}
    }
\end{floatrow}
\end{figure}

The rotated MNIST data set~\citep{Larochelle2007}, examples of which are shown in Fig.~\ref{fig:mnist}, has been used as a benchmark for several previous works on rotation invariance.
We compared our overall method of RiCNN, as well as G-CNN with the integration of the 2D-DFT, against results reported in~\citep{Cohen2016}.
In our analysis, the only other change we made from the reported architecture for G-CNN was to reduce the number of filters for each layer to 7, to offset the addition of the 2D-DFT.
Our RiCNN architecture used the same number of convolutional layers and filter size, except that 20 filters were used at the transition layer and we added a fully-connected layer of 10 nodes between the DFT and the output layer.
As in previous works, to tune the parameters of each method, we first trained various models on a set of 10,000 images and then selected the best model based on the accuracy on a separate validation set of 5,000 images.
The results of these two approaches, along with the reported results of the other state-of-the-art, on a held-out set of 50,000 test images are shown in Table~\ref{tab:rotated_mnist}.
RiCNN outperforms the standard CNN and performs comparably to G-CNN, while requiring less computation and storage.
Notably, replacing the pooling operations in G-CNN with the 2D-DFT provides a meaningful improvement for this particular problem, demonstrating the value of incorporating mutual rotational information between filters when encoding invariance.
We also note that, depending upon the style of writing, the digits six and nine are often indistinguishable under rotation and that it is likely that perfect accuracy is not attainable.

\subsection{Application to Synthetic Biomarker Images}

In order to explicitly control the manifestation of rotational invariance, as well as the types and degree of inter-class and intra-class variation, in a classification task, we created a set of synthetic images, in this case to emulate real-world microscopy images of biological signals.
We sought a generative model for which classes could not be discriminated by trivial, rotation-invariant features and did not express local rotational invariance, which might be easier to learn by a CNN without the need for global invariance.
To achieve this, we used Gaussian-mixture models (GMMs) to generate spatial patterns for each class, which represent unique spatial distributions of synthetic biomarkers within a cell.
This approach has been used similarly in other work to generate synthetic models of cells and their subcellular protein objects~\citep{Zhao2007}.

Mathematically, each class $k\in\left\{1,\ldots,K\right\}$ is described by the parameters $\Lambda$ of the GMM: $\Lambda_{k} = \left\{\vec{\mu}_{g}, \Sigma_{g}\right\}^{G}_{g=1}$,
where $\vec{\mu}_{g}\in \left[-1,1\right]^{2}$ and the number of Gaussians per class is a parameter $G$ of the data set.
For simplicity, we consider the image $I\colon \mathbb{R}^{2}\to \mathbb{R}$ to be nonzero only over a region slightly larger than the $\left[-1,1\right]^{2}$ box, so that it captures the majority of points generated by the Gaussians.

To generate a sample image from the generating distribution, first, a constant background intensity is set for the image according to $b\sim \operatorname{Exp}(0, \lambda_{B})$, so $I(p)=b$, $\forall p\in\mathbb{R}^{2}$.
Then a random angle $\theta\sim\operatorname{Uniform}\left[0,2\pi\right]$ is drawn to determine the rotation of the image.
The mean $\vec{\eta}_{g}\sim\mathcal{N}\left(\vec{\mu}_{g}, \Sigma\right)$ for each Gaussian of the class is drawn from an underlying Gaussian with mean $\vec{\mu}_{g}$, which introduces some small jitter of the relative locations of the Gaussians.
A number $n_{g}\sim \mathcal{N}(\mu_{n,g}, \sigma_{n})$ of points $\left\{p\right\}$ in $\left[-1,1\right]^{2}$, which vary for each Gaussian, are drawn from this Gaussian according to $p\sim\mathcal{N}\left(R\vec{\eta}_{g}, R\Sigma_{g}R^{-1}\right)$,
where the realized mean and covariance have been rotated by $\theta$ by the rotation matrix:
\begin{equation}
    R = \begin{bmatrix}
        \cos(\theta) & -\sin(\theta)\\
        \sin(\theta) & \cos(\theta)
    \end{bmatrix}.
\end{equation}
For each point $p$, its corresponding intensity value is drawn according to $I(p)\sim \operatorname{Uniform}\left[\mu_{I}-m_{I}, \mu_{I}+m_{I}\right]$, replacing the background value.
Having drawn all of the points, the image is smoothed with a Gaussian kernel with variance $\sigma_{s}$ to emulate the point-spread function of the imager and pixel noise is added: $I(p) = I(p) + \operatorname{Exp}(0, \lambda_{I})$.
To simulate camera jitter, the image is translated by a random offset of up to three pixels.

\begin{figure}[t]
  \centering
  \begin{minipage}{0.25\textwidth}
  \begin{minipage}{\textwidth}
      \centering
      \includegraphics[width=0.3\textwidth]{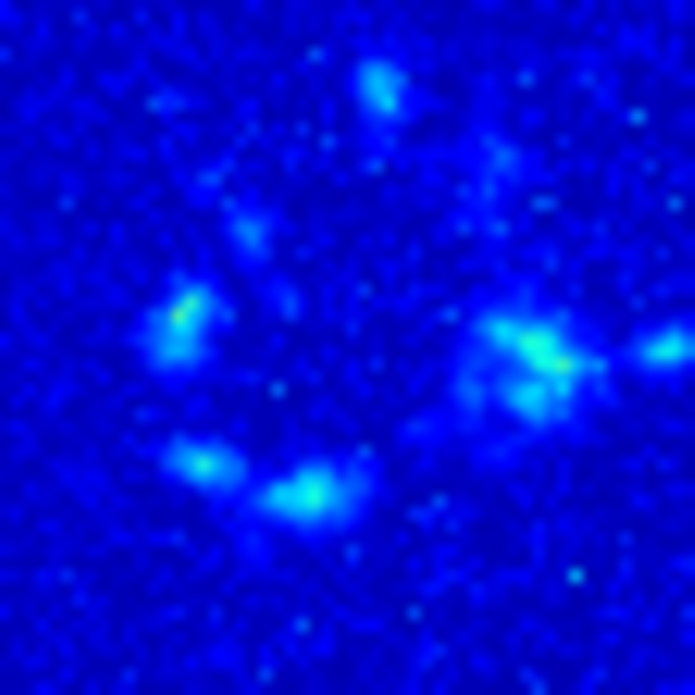}
      \includegraphics[width=0.3\textwidth]{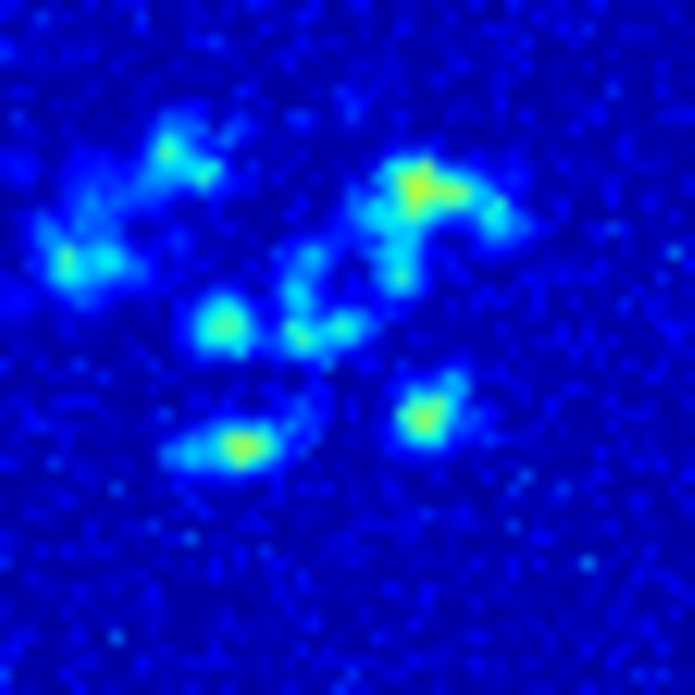}
      \includegraphics[width=0.3\textwidth]{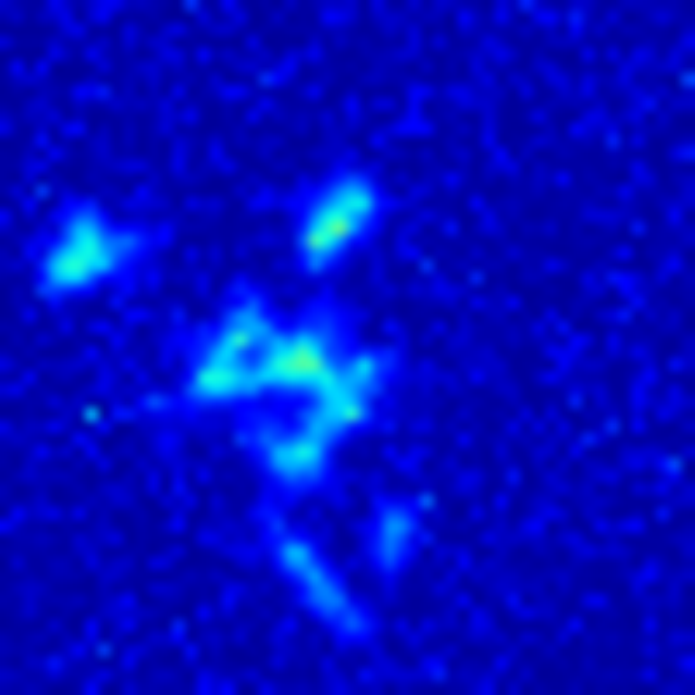}\\
      \includegraphics[width=0.3\textwidth]{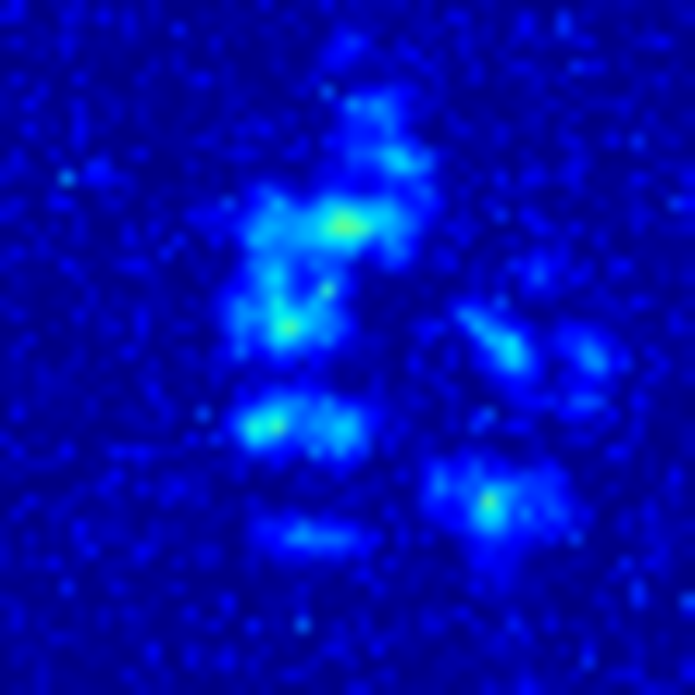}
      \includegraphics[width=0.3\textwidth]{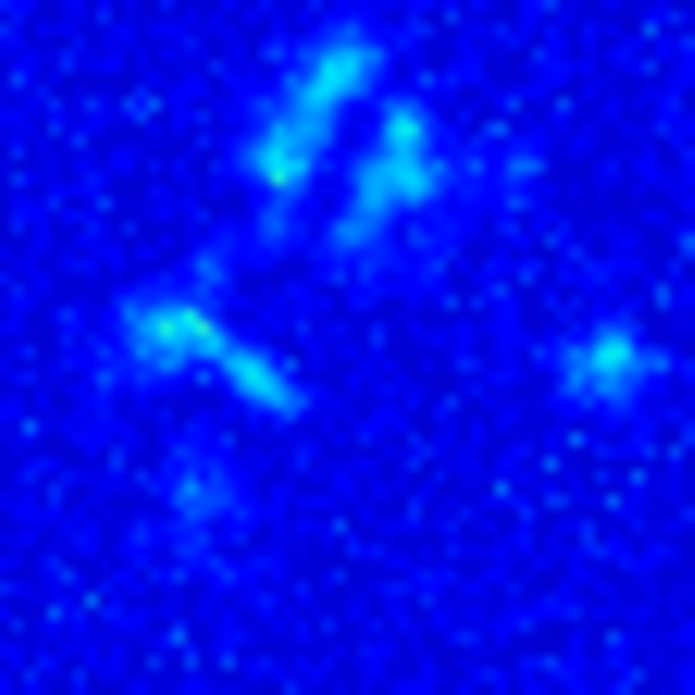}
      \includegraphics[width=0.3\textwidth]{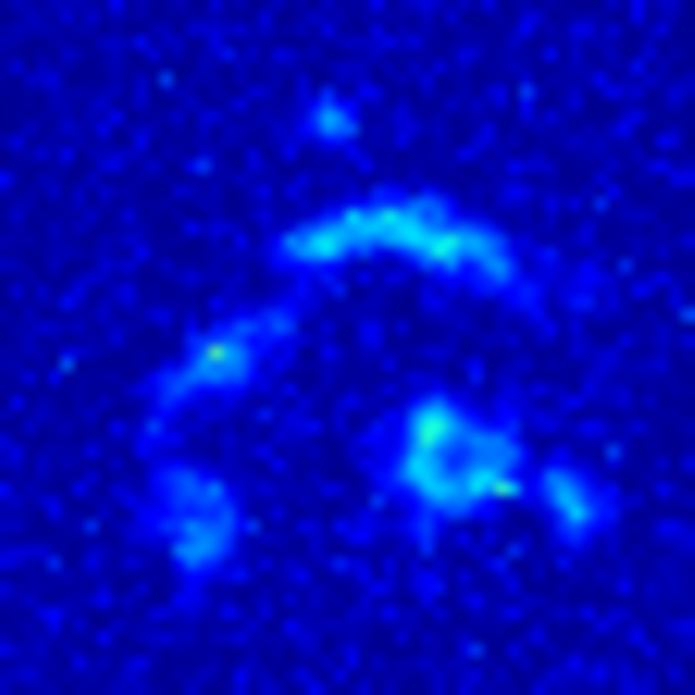}
      \subcaption{Example images of six of the 50 classes, showing the inter-class variation.}
      \label{fig:synthetic_examples}
  \end{minipage}
  \begin{minipage}{\textwidth}
      \centering
      \includegraphics[width=0.3\textwidth]{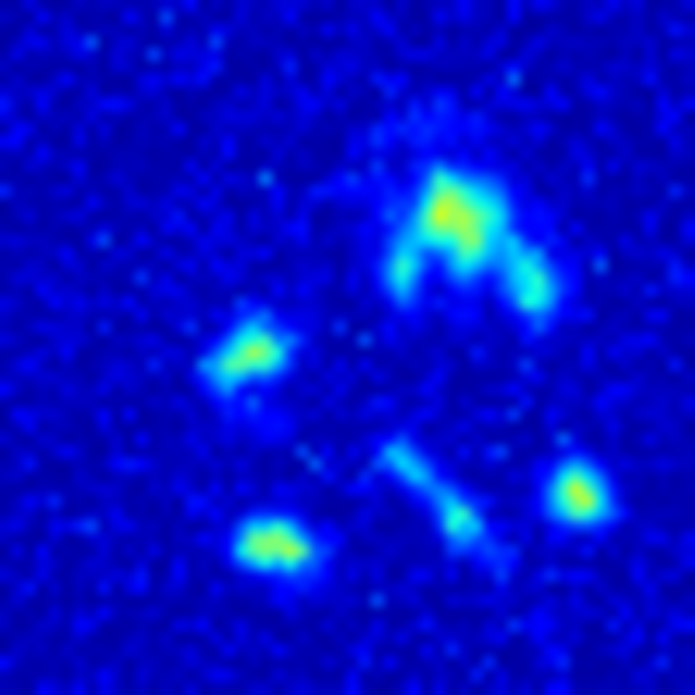}
      \includegraphics[width=0.3\textwidth]{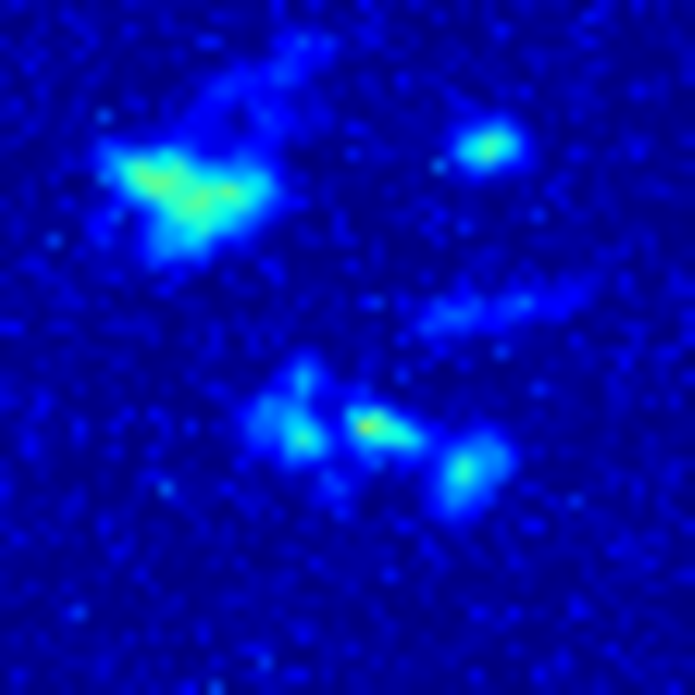}
      \includegraphics[width=0.3\textwidth]{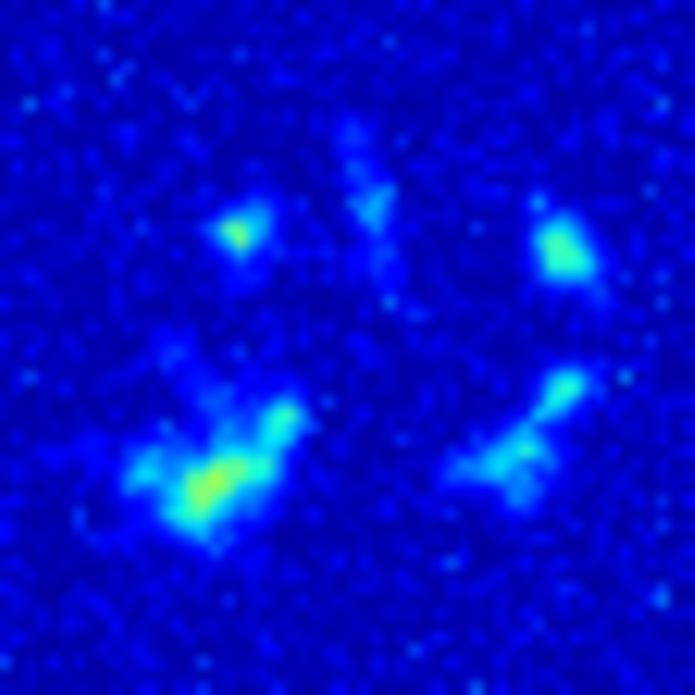}\\
      \includegraphics[width=0.3\textwidth]{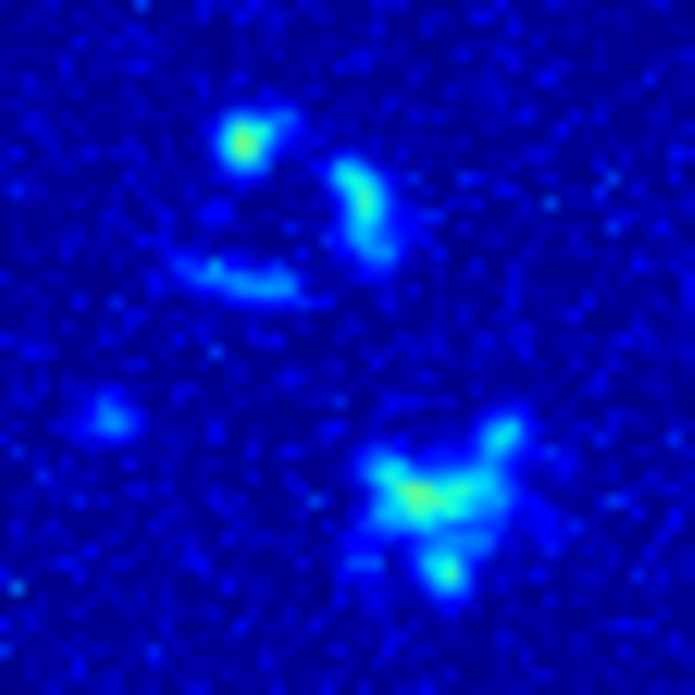}
      \includegraphics[width=0.3\textwidth]{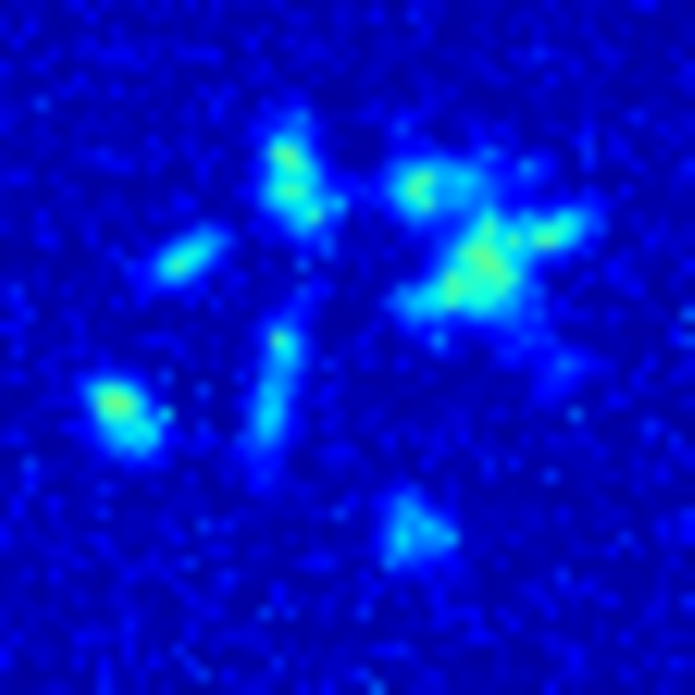}
      \includegraphics[width=0.3\textwidth]{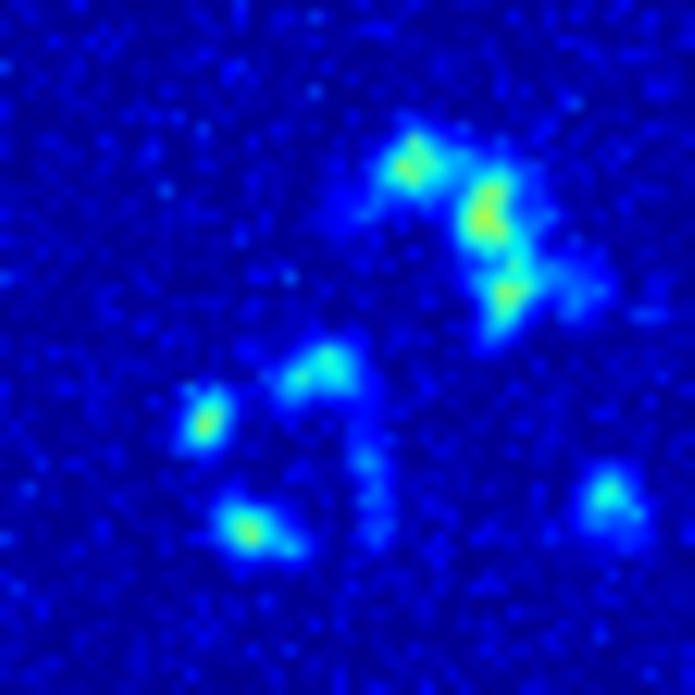}
      \subcaption{Rotated examples from a single class of rotations from 0 to $2\pi$ (left to right, top to bottom), showing the intra-class variation.}
      \label{fig:synthetic_rotation_examples}
  \end{minipage}
  \end{minipage}
  \begin{minipage}{0.36\textwidth}
  \begin{minipage}{\textwidth}
      \centering
      \includegraphics[width=\textwidth]{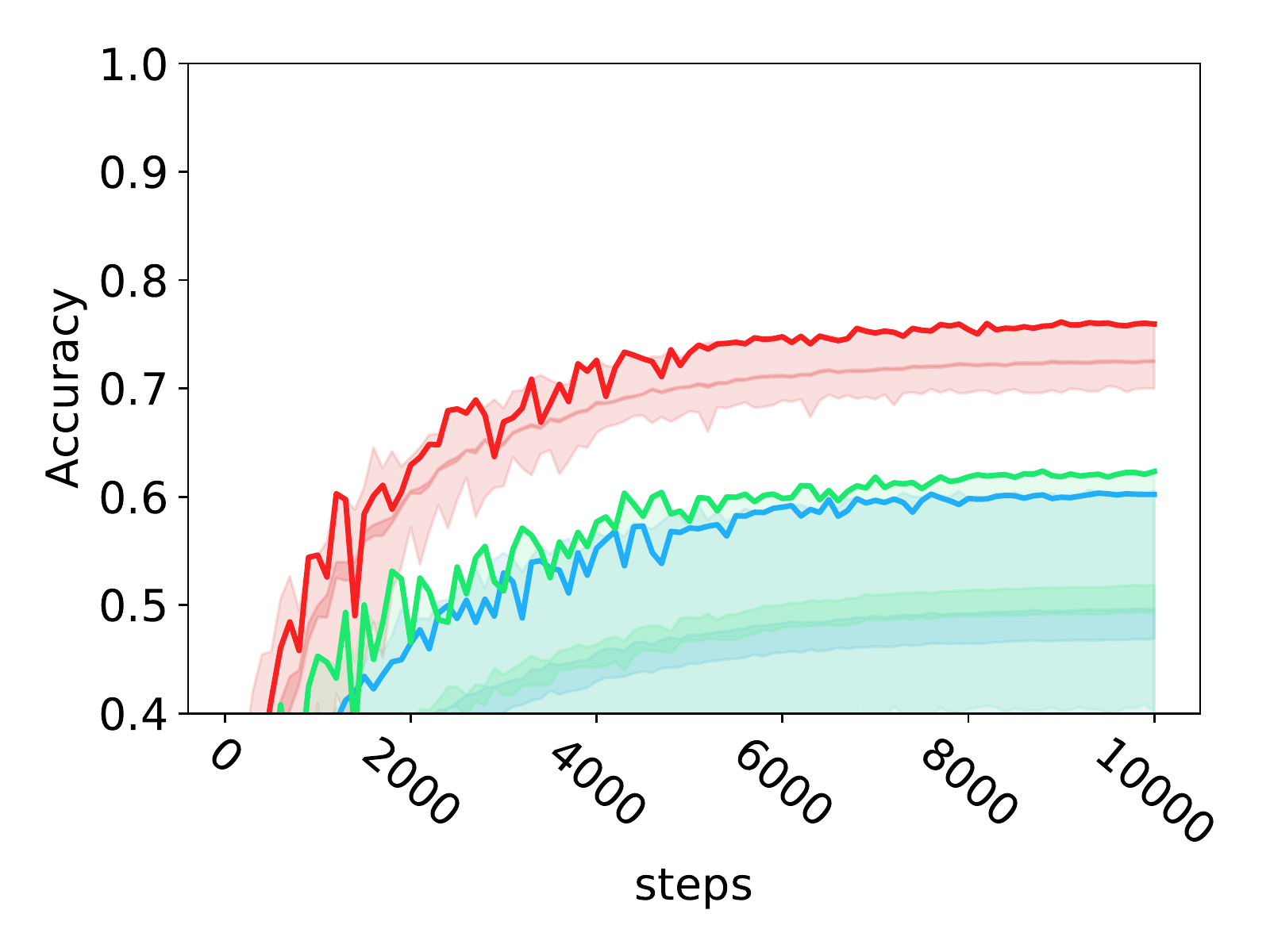}
      \subcaption{$N=10$.}
      \label{fig:synthetic_avg_acc_10}
  \end{minipage}
  \begin{minipage}{\textwidth}
      \centering
      \includegraphics[width=\textwidth]{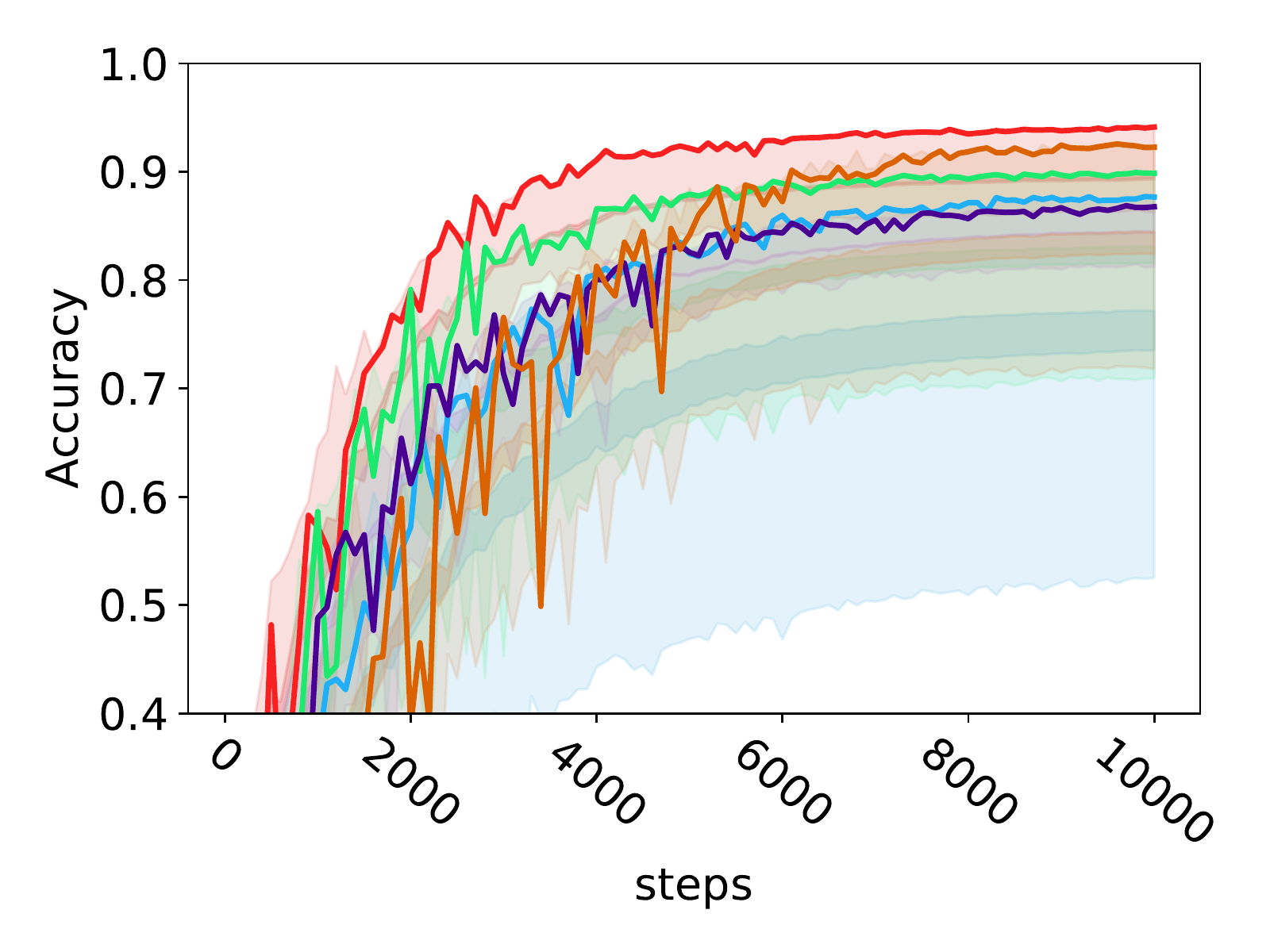}
      \subcaption{$N=50$.}
      \label{fig:synthetic_avg_acc_50}
  \end{minipage}
  \end{minipage}
  \begin{minipage}{0.36\textwidth}
  \begin{minipage}{\textwidth}
      \centering
      \includegraphics[width=\textwidth]{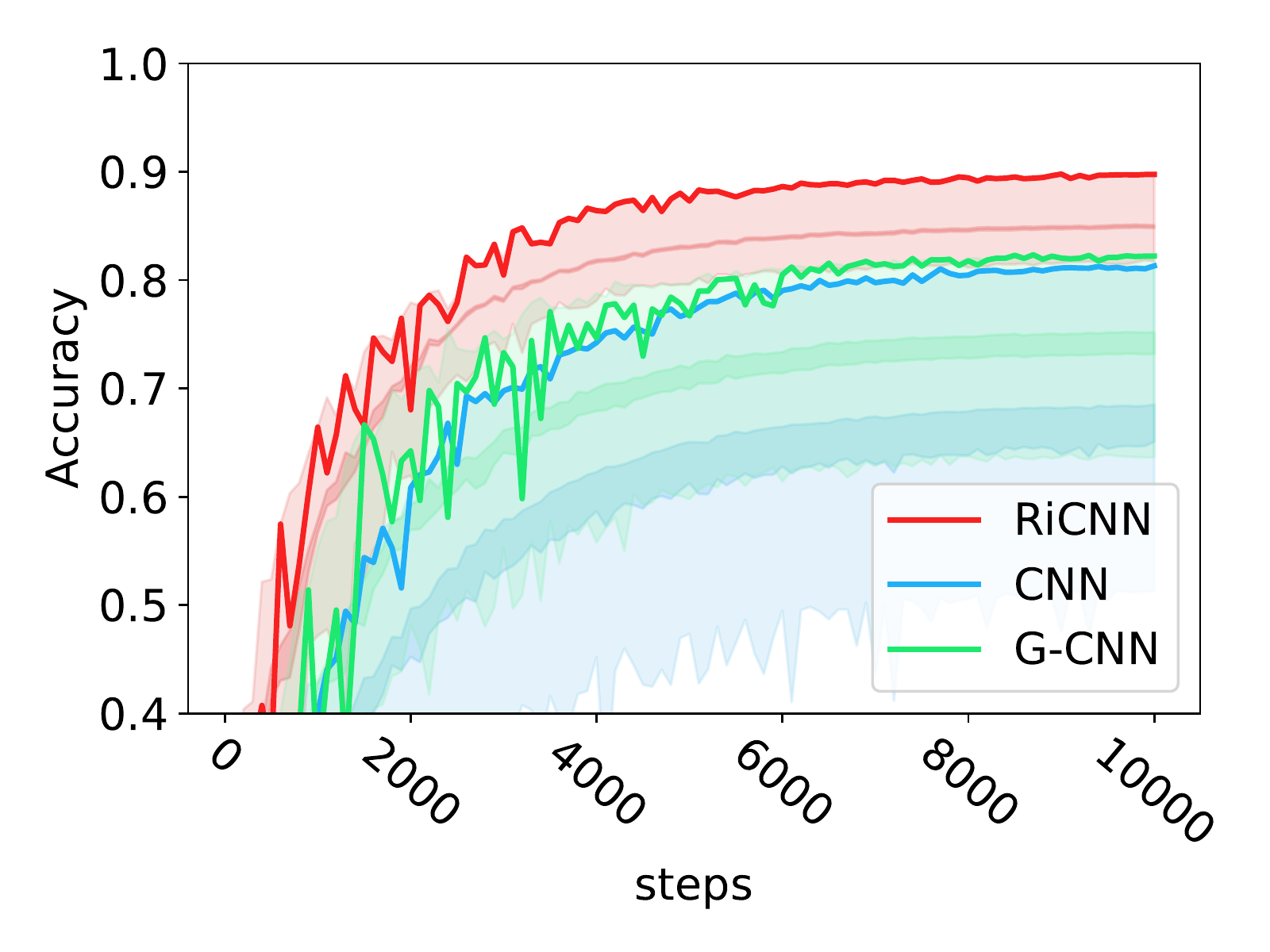}
      \subcaption{$N=25$.}
      \label{fig:synthetic_avg_acc_25}
  \end{minipage}
  \begin{minipage}{\textwidth}
      \centering
      \includegraphics[width=\textwidth]{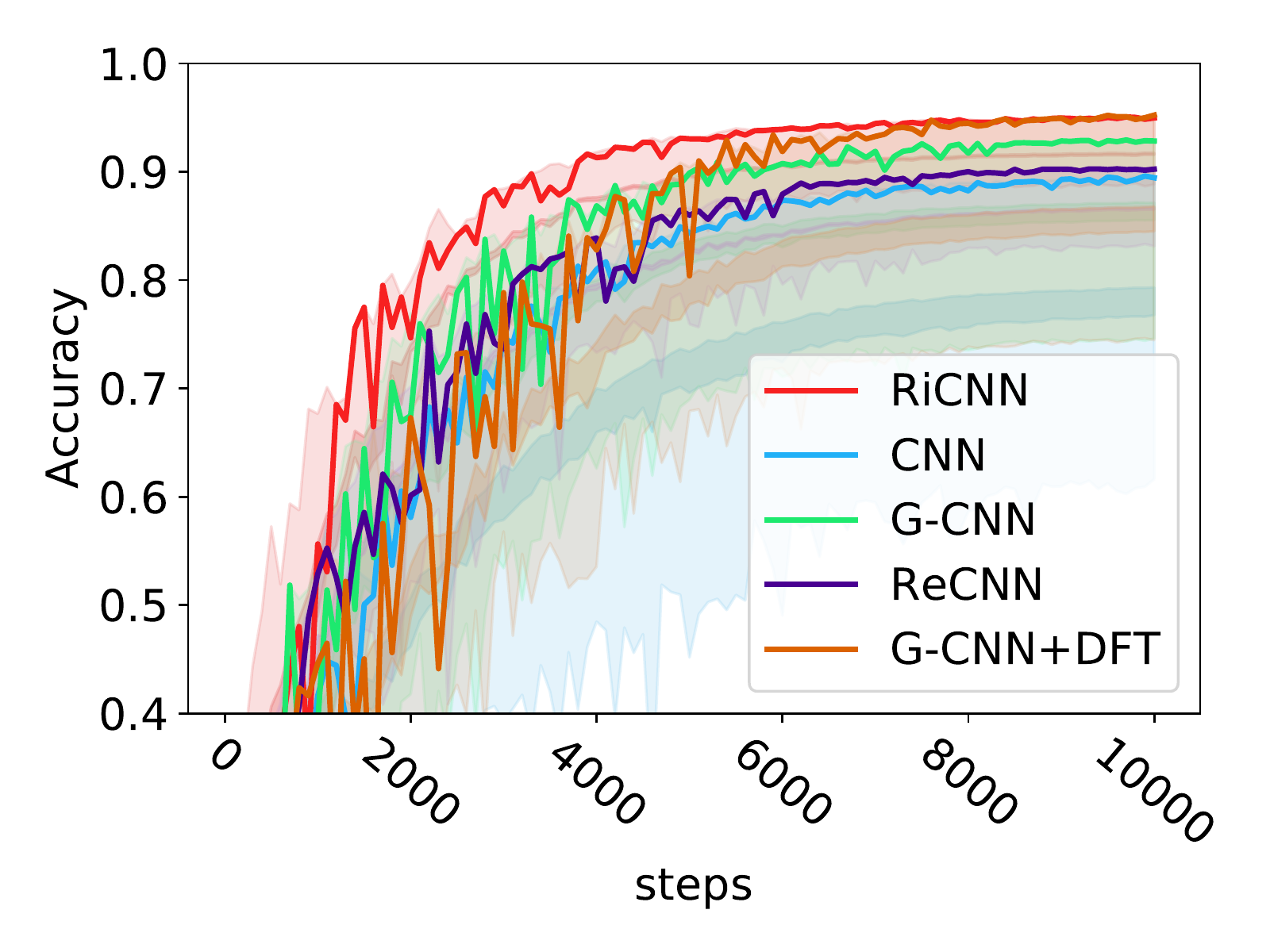}
      \subcaption{$N=100$.}
      \label{fig:synthetic_avg_acc_100}
  \end{minipage}
  \end{minipage}
  \caption{Comparison of RiCNN, ReCNN, G-CNN, G-CNN+DFT, and a standard CNN on the GMM synthetic biomarker images. (a),(b) Example images, shown as heat maps for detail, showing inter- and intra-class variation. (c)-(f) Testing classification accuracy of RiCNN, GCNN, and a standard CNN over training steps on synthetic GMM images, with varying numbers of training examples per class, denoted by $N$.}
  \label{fig:synthetic}
\end{figure}

Example synthetic images from across and within classes are shown in Fig.~\ref{fig:synthetic_examples} and Fig.~\ref{fig:synthetic_rotation_examples}, respectively.
For our experiment, we defined 50 distribution patterns and generated 10, 25, 50, and 100 examples per class for training and 200 examples per class for testing.
Each class was defined by a mixture of ten Gaussians. 
The image size was 50 pixels.
A batch size of 50 examples, a learning rate of $5\times 10^{-3}$, and a weight decay $\ell_{2}$ penalty of $5\times 10^{-4}$ were used during training.
To help all methods, we augmented the training data by rotations of $\frac{\pi}{2}$ and random jitter of up to three pixels, as was done during image generation.

Classification accuracies on the test set over training steps for various numbers of training samples, denoted by $N$, for RiCNN, G-CNN, and a standard CNN are shown in Figs.~\ref{fig:synthetic_avg_acc_10}-\ref{fig:synthetic_avg_acc_100}.
In addition, for the sets of 50 and 100 training examples per class, we compared the performance of our rotation-equivariant convolutional neural network (ReCNN) without the 2D-DFT transition, as well as G-CNN with the 2D-DFT.
A variety of configurations were trained for each network, and each configuration was trained three times.
The darkest line shows the accuracy of the configuration that achieved the highest moving average, with a window size of 100 steps, for each method.
The spread of each method, which is the area between the point-wise maximum and minimum of the error, is shaded with a light color, and three standard-deviations around the mean is shaded darker. 

For training sets of 10, 25, 50, and 100 images per class, we observe a consistent trend of RiCNN outperforming G-CNN, which in turn marginally outperforms the CNN, both in overall accuracy and in terms of the number of steps required to attain that accuracy.
Additionally, the spread of RiCNN is mostly above even the best performing models of G-CNN and the CNN.
This demonstrates that an instance of RiCNN will outperform other methods even if the best set of hyperparameters has not been chosen.

Without including the 2D-DFT, ReCNN performs comparably to a standard CNN, but it has the advantage of requiring significantly fewer parameters.
It is notable that, again, including the 2D-DFT increases the performance of G-CNN, to a comparable level with RiCNN in fact, though it does not train as quickly.

\subsection{Application to Protein Localization in Budding Yeast Cells}

\begin{figure}[t]
  \centering
  \begin{minipage}{0.25\textwidth}
      \centering
      \includegraphics[width=0.4\textwidth]{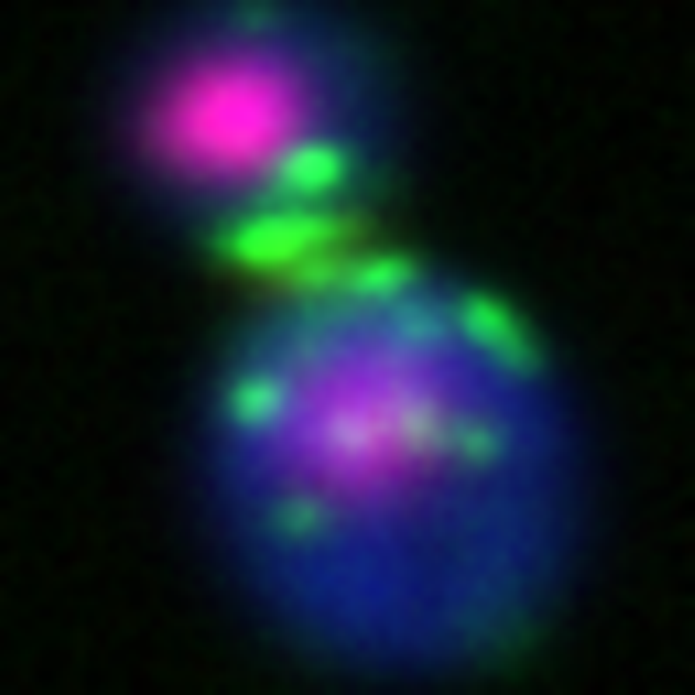}
      \includegraphics[width=0.4\textwidth]{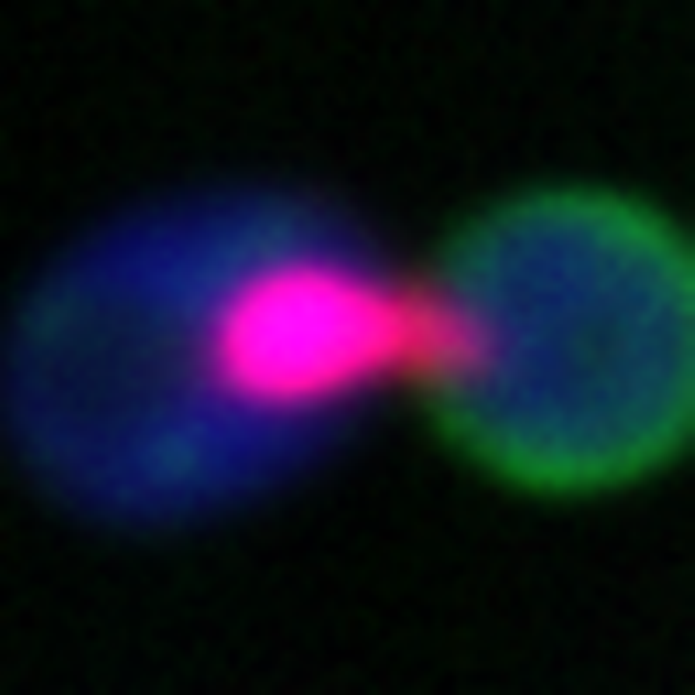}\\
      \includegraphics[width=0.4\textwidth]{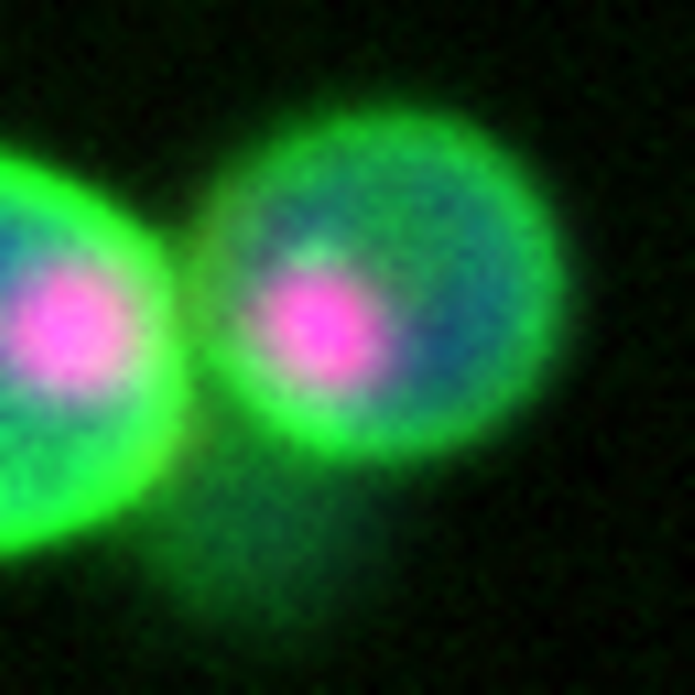}
      \includegraphics[width=0.4\textwidth]{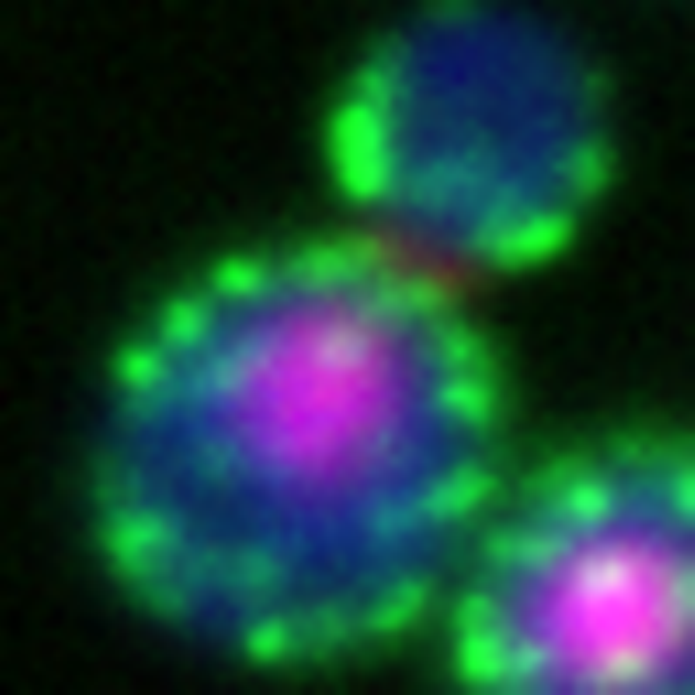}
      \subcaption{Example images of cells of four of the 22 yeast phenotypes from \citet{Kraus2017}.}
      \label{fig:yeast_phenotypes}
  \end{minipage}
  \begin{minipage}{0.36\textwidth}
  \begin{minipage}{\textwidth}
      \centering
      \includegraphics[width=\textwidth]{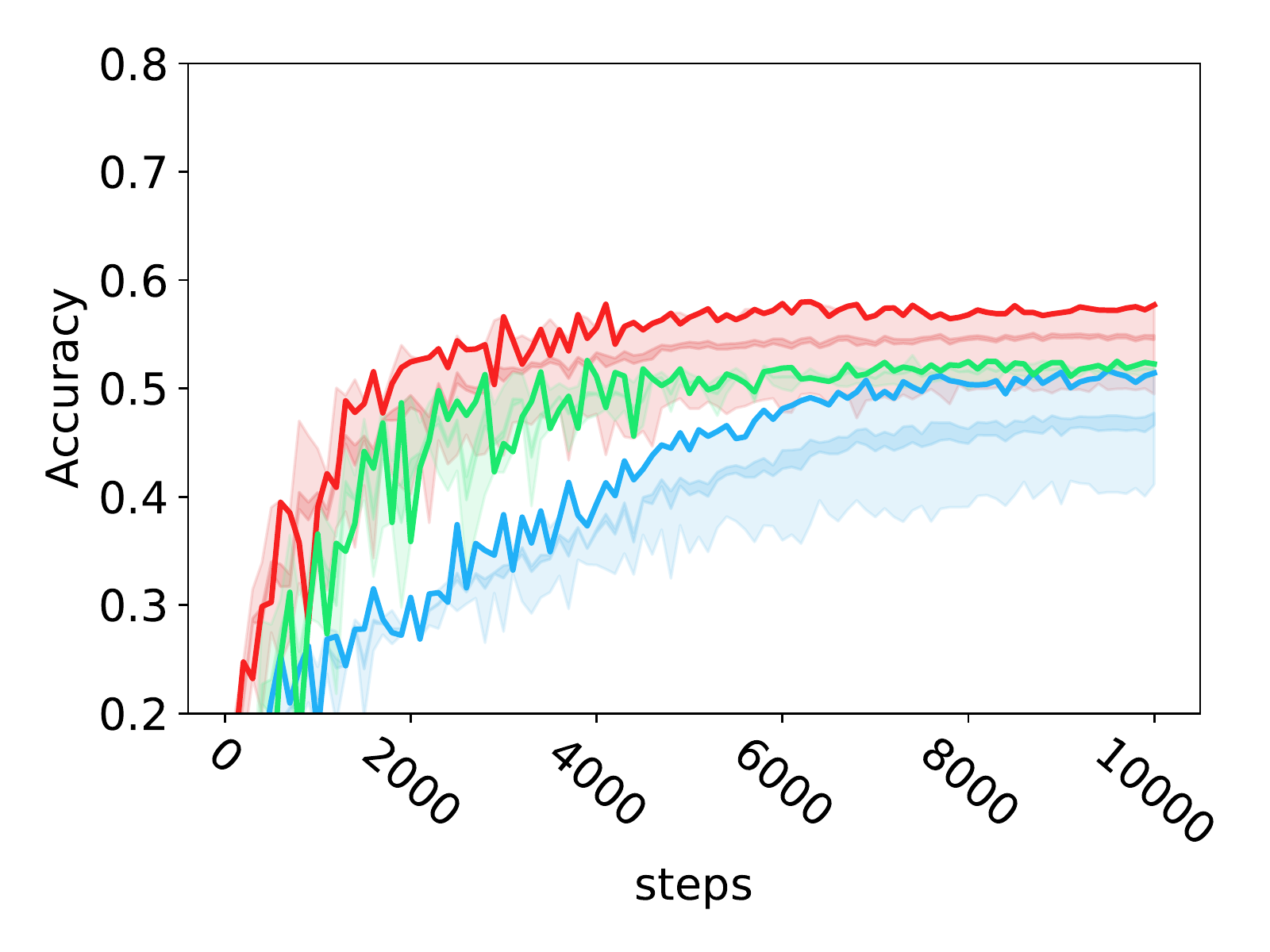}
      \subcaption{$N=50$}
      \label{fig:yeast_50}
  \end{minipage}
  \end{minipage}
  \begin{minipage}{0.36\textwidth}
  \begin{minipage}{\textwidth}
      \centering
      \includegraphics[width=\textwidth]{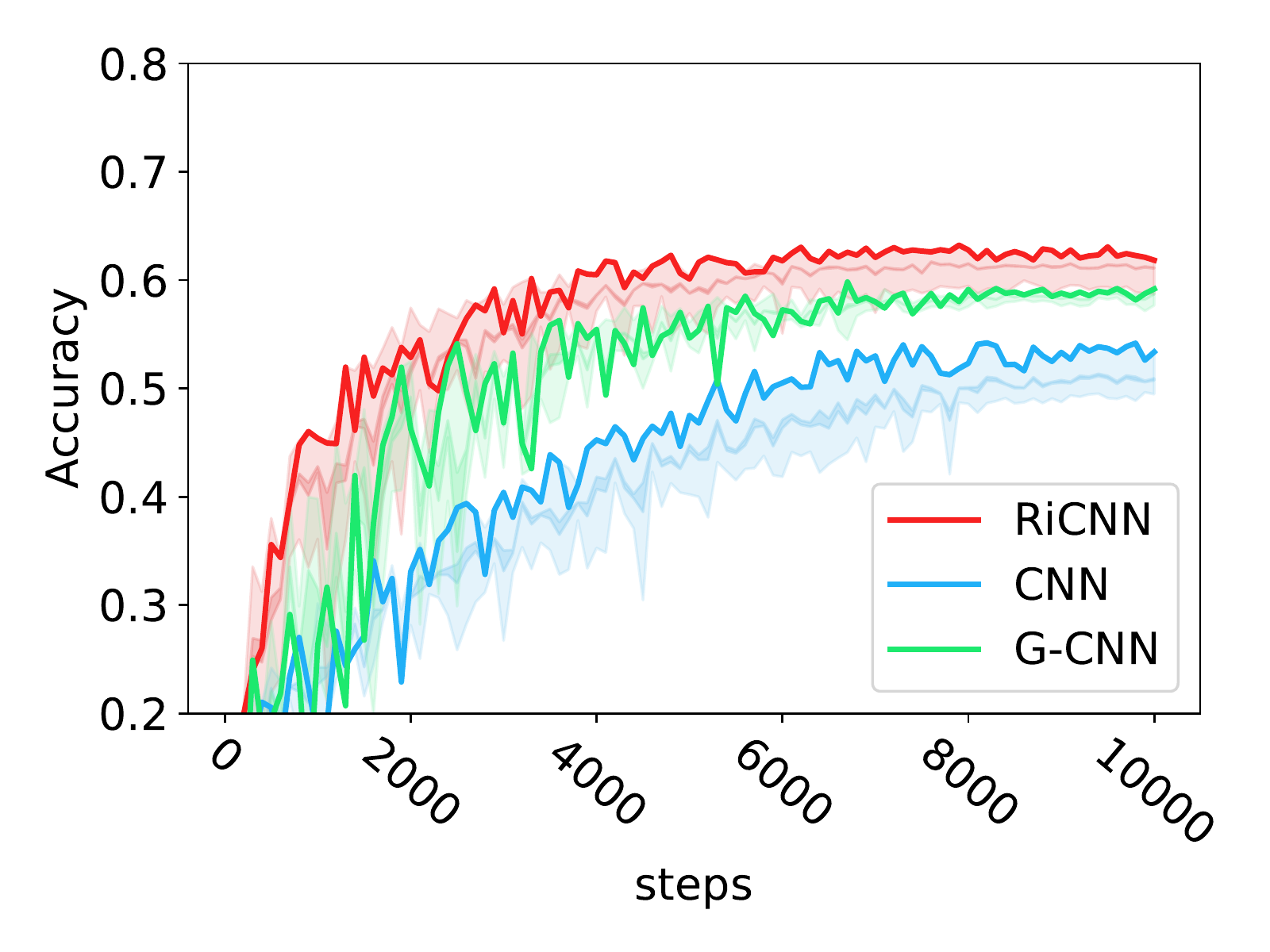}
      \subcaption{$N=100$}
      \label{fig:yeast_100}
  \end{minipage}
  \end{minipage}
  \caption{Comparison of RiCNN, G-CNN, and a standard CNN for classifying budding yeast cell phenotypes. (a) Example images from phenotype classes. (b)-(c) Testing accuracy of RiCNN, GCNN, and a standard CNN for classifying budding yeast phenotypes over training steps, with varying numbers of training examples per class, denoted by $N$.}
  \label{fig:yeast}
\end{figure}
%

Having demonstrated the merit of RiCNN on synthetic images, we further extended our analysis to real biomarker images of budding yeast cells~\citep{Kraus2017}, shown in Fig.~\ref{fig:yeast_phenotypes}.
Each image consists of four stains, where blue shows the cytoplasmic region, pink the nuclear region, red the bud neck, and green the protein of interest.
The classification for each image is the cellular subcompartmental region in which the protein is expressed, such as the cell periphery, mitochondria, or eisosomes.
This particular task is significantly more challenging than the classifying the synthetic images, since differences between phenotypes can be subtle.

Fig.~\ref{fig:yeast} shows the results of using RiCNN, G-CNN, and a standard CNN to classify the protein localization for each image.
We used the same architecture as reported in~\citep{Kraus2017} for all methods, except that we removed the last convolutional layer and reduced the number of filters per layer by roughly half for RiCNN and G-CNN, to offset for encoding of equivariance and invariance.
The same training parameters and data augmentation were used as for the synthetic data, except that a dropout probability of 0.8 was applied at the final layer and the maximum jitter was increased six pixels, since we could not control as well for proper centering. 
For each method, several iterations were run and the spread and the best performing model are shown.
Again, RiCNN outperforms G-CNN and a standard CNN, when the number of training examples per class is either 50 or 100 (see Fig.~\ref{fig:yeast}b-c),
demonstrating that the gains of the 2D-DFT and proposed convolutional layers translate to real-world microscopy data.
We note that the best reported algorithm that did not use deep learning, called ensLOC~\citep{Chong2015,Koh1223}, was only able to achieve an average precision of 0.49 for a less challenging set of yeast phenotypes and with $\sim$20,000 samples, whereas all runs of RiCNN achieved an average precision of between 0.60 - 0.67 with $\sim$10\% of the data used for training.

\section{Conclusion}

In this paper, we have demonstrated the effectiveness of enforcing rotation equivariance and invariance in CNNs by means of the proposed convolutional layer and the 2D-DFT.
In particular, we have demonstrated the utility of the 2D-DFT for encoding invariance across two different equivariant approaches to convolution.
We have also demonstrated the proposed rotation-equivariant convolutional layer can contribute to better accuracy in classification tasks while requiring minimal additional computation or storage.
We believe that the proposed enhancements to the standard CNN will have much utility for future applications in relevant problem settings, in particular, molecular and cellular imaging.
Especially as high-throughput imaging studies become more specific, our new method may facilitate new discoveries in detecting rare cellular events and accurate classification of complex phenotypes. 

\clearpage
\bibliographystyle{unsrtnat}
\bibliography{cellphenotyping_ricnn}

%
%
%
\end{document}